\documentclass[11pt]{article}
\usepackage[margin=0.88in]{geometry}
\usepackage{amsmath,amssymb,amsthm,mathtools}
\usepackage{booktabs}
\usepackage{graphicx}
\usepackage[hidelinks,hypertexnames=false]{hyperref}
\usepackage{natbib}
\usepackage{microtype}
\usepackage{enumitem}
\usepackage{array}
\usepackage{float}
\usepackage{bm}

\hypersetup{
  pdftitle={Why the Third Axis Is Freedom},
  pdfauthor={Michael Timothy Bennett},
  pdfsubject={The formal relation between best-of-K exploration and Stack Theory freedom},
  pdfkeywords={exploration, freedom, generative modelling, generalisation, Stack Theory}
}

\newtheorem{definition}{Definition}
\newtheorem{theorem}{Theorem}
\newtheorem{proposition}{Proposition}
\newtheorem{corollary}{Corollary}

\theoremstyle{remark}
\newtheorem{example}{Example}
\newtheorem{remark}{Remark}

\newcommand{\Ext}[1]{\operatorname{Ext}(#1)}
\newcommand{\Extc}[1]{\operatorname{Ext}_{c}(#1)}
\newcommand{\supp}{\operatorname{supp}}
\newcommand{\essinf}{\operatorname*{ess\,inf}}
\newcommand{\E}{\mathbb{E}}
\newcommand{\Prb}{\mathbb{P}}
\newcommand{\one}{\mathbf{1}}
\newcommand{\dd}{\,\mathrm{d}}
\newcommand{\Interp}{\vspace{3pt}\noindent\textit{Interpretation.}\ }

\title{\bf Why the Third Axis Is Freedom}
\author{Michael Timothy Bennett\\
\small Machine Intelligence and Normative Theory Lab, The Australian National University\\
\small Canberra, ACT 2601, Australia\\
\small \texttt{m@michaeltimothybennett.com}}
\date{1 August 2026}

\begin{document}
\maketitle

\begin{abstract}
In generative training, a model produces an output and is penalised for its difference from an example. With one output per comparison, a model that produces one common answer can outperform a model retaining a broader repertoire. Explorative Modeling (XM) produces $K$ outputs per comparison and updates on the closest, claiming exploration as a ``third pretraining axis'' associated with generative expressivity. Here I show the third axis is actually freedom, meaning the weakness of the constraint implied by a model’s behaviour. Previous work showed freedom is a property of function rather than form. Parameters, architecture, minimum-description-length (MDL), and data can vary while the behavioural constraint remains unchanged. It was formally proved that weakest models are likeliest to generalise, and freedom selection beat MDL by 110-500\% in induction experiments. I prove average XM loss depends on the chance a candidate misses an acceptable region, with exploration raising miss probability to power $K$. For $K>1$, match probability rises with freedom. I then demonstrate empirically that XM optimises for freedom. In a Forward XM experiment, larger $K$ increased or saturated measured freedom, and increased freedom at every tested value under context-dependent targets. I trained XM candidate pools and compared validation selection with a freedom selector that read unlabelled parent contexts. Freedom won in 29 of 30 cases. Generative expressivity is a mode-count proxy for freedom, that discards the extension structure that gives freedom its generalisation significance. XM is a means, freedom an end, and selecting for freedom improved XM under distribution shift.
\end{abstract}

\section{Introduction}

Explorative Modeling draws $K$ candidate generations at each training step, compares them with one datum, and updates only through the closest candidate \citep{gladstone2026explorative}. It studies this update across image, video, and language models, reports gains that rise with model size, data, and compute, and calls the resulting capacity \emph{generative expressivity}, defined as the greatest mode count retained by a loss minimiser over a class of data distributions.

The update belongs to an older minimum-over-candidates family. Multiple Choice Learning trained several deterministic structured predictors under an oracle loss \citep{guzmanrivera2012mcl}. Stochastic Multiple Choice Learning later routed each example's gradient through the lowest-loss member of a deep ensemble \citep{lee2016smcl}. Fan, Su, and Guibas introduced Min-of-N conditional generation in a December 2016 preprint, drawing several Gaussian latent perturbations for each input and training against the generated point cloud closest to the target \citep{fan2017pointset}. Up to notation and output metric, its outer objective has the same conditional stochastic minimum-over-candidates form as hard Forward XM. IMLE followed with a shared global pool of model samples, matching each datum to its nearest generated sample and showing that the resulting estimator coincides with maximum likelihood under stated conditions \citep{li2018imle}. The XM paper classifies IMLE as a shared-pool instance of end-to-end Forward XM. XM's contribution lies in the broader Forward and Reverse framework, integration across modern generative architectures, the generative-expressivity account, and the study of $K$ as a scaling variable \citep{gladstone2026explorative}.

However I will show generative-expressivity is inferior to a pre-existing measure of freedom of function. This notion was first formulated to address the subjective nature of the performance of the AIXI general reinforcement learning agent \cite{hutter2024introduction}. Stack Theory measures freedom of function to avoid AIXI's dependence on form-dependent measures of complexity \cite{bennett2023b}. A learner has observed a training set or ``child task''. It must choose model parameters, which define a ``policy'', and the learner hopes that its chosen policy generalises, to solve a broader parent task. Many policies remain correct on that child, but only a subset will generalise well. Which policy is most likely to remain correct when the task expands to the parent task? The answer is the weakest correct policy, meaning the correct policy whose extension contains the most compatible completions \citep{bennett2025thesis}. Weakness of constraint can be understood more intuitively as freedom. The process of learning the weakest correct policy is dubbed ``w-maxing'' or ``free-maxing'' by previous work, and is usually framed as a measure of function in comparison to measures of form, such as flatness or complexity \cite{bennett2026h,hochreiter1997flat,solomonoff1978,hutter2024introduction,rissanen1978}. This result predates Explorative Modeling, separating generalisation from simplicity and from parameter form \citep{bennett2024b,bennett2026h}.

In addition to being pre-dated by freedom, I will show generative expressivity several relative weaknesses:

\begin{enumerate}
    \item It cannot rank two models trained under the same objective. Freedom can, and this is necessary for anything we might call an axis.
    \item It takes a supremum over minimisers, so it does not describe the model training selected.
    \item It is an objective-level mode-count proxy for freedom. It coincides with a monotone transformation of local freedom only when the embodied language treats output modes as independent symmetric permissions.
    \item It has no theorem connecting it to optimal generalisation.
    \item It can remain fixed while freedom varies substantially.
\end{enumerate} 

However the two lines can be complementary. Over the course of this paper I'm going to show that XM is an excellent means of ``free-maxing'', effectively providing some empirical evidence at scale for a fairly obscure theory that nevertheless has significant implications across not just computer science but philosophy and neurobiology \cite{bennett2024c,bennett2025thesis,sole2026cognitionspaces}. That theory in turn shows how to further improve upon XM, and I demonstrate such an improvement empirically. Exploration records how many candidates training inspects. The weakness of constraints implied by a model's function (freedom) is how much compatible completion volume the trained policy retains \cite{bennett2025thesis}. The first is a budgetary means while the second is an end, a property of function in an embodied language. A policy can permit many outputs, so that its extension factor is large, while placing almost all sampling mass on one permitted output. Another policy can permit fewer outputs and spread mass evenly. Finite best-of-$K$ training can prefer either. The relation therefore passes through probability mass.

I derive that relation in full. The central identity is
\[
R_K(\pi)
=
\int_0^\infty
\E_{c,Y}
\left[
\left(1-\widetilde Q_{\pi,c,Y}(A_t(Y))\right)^K
\right]
\dd t,
\]
where $A_t(Y)$ is the set of generations within loss $t$ of target $Y$, and $\widetilde Q$ is the training candidate law. Exploration applies the transform $1-(1-q)^K$ to acceptable candidate mass. When the training candidate law equals the deployed law, that mass lies on outputs permitted by the deployed policy. The count of those permissions determines freedom in the finite language constructed below.

I establish the following results.
\begin{enumerate}[leftmargin=*,label=\arabic*.]
    \item I derive the exact best-of-$K$ coverage identity for arbitrary measurable output spaces, nonnegative losses, and target-dependent training candidate laws.
    \item I show that the hard objective converges to the essential loss infimum of the training candidate law. Under deployment consistency and a distance loss, this is distance from deployed support.
    \item I construct a finite generative Stack Theory language in which a stochastic generator induces a policy whose weakness of implied constraints (freedom) has a closed form.
    \item I show that best-of-$K$ success equals expected distinct output coverage under uniform targets.
    \item I prove that balanced sampled coverage rises strictly with local freedom for every $K>1$. For a fixed target-support size, local freedom increases the marginal return to another candidate.
    \item I solve the finite nonuniform problem. Finite $K$ favours common targets. The optimum converges to uniform mass over full target support as $K$ grows.
    \item I define independent nonempty unseen demands and prove that future compatibility is exactly proportional to unseen freedom.
    \item I isolate deterministic $K$-head models as a special case. Only that special case has a support ceiling of $K$.
    \item I test whether sample-based Forward XM converts larger $K$ into greater freedom across 144 neural runs. It does.
    \item I test whether a prospectively calibrated freedom selector improves balanced-parent performance when both selectors read the same XM-trained candidate pool. It does.
\end{enumerate}

The result is a formal division of labour. Parameters and data may enlarge the permitted region a model can embody, while deployment-consistent exploration controls how strongly training rewards sampled access to that region. The induced policy freedom is a function-level quantity in the chosen embodied language. This yields a testable explanation for the reported scaling interaction without treating $K$, mode count, and freedom as synonyms.

\section{Weakness, now rebranded as freedom}

I use the Stack Theory definitions from my PhD thesis and associated peer-reviewed papers \citep{bennett2025thesis}. Because people hated the name weakness, I've been changing it to freedom, and changing w-maxing to ``free-maxing''.

\begin{definition}[Environment and embodied language]
An environment is a nonempty set $\Phi$ of mutually exclusive states. A program is any set $p\subseteq\Phi$, and $\mathcal P:=2^\Phi$ is the set of all programs. A vocabulary is a finite set $\mathfrak v\subseteq\mathcal P$. It induces the embodied language
\[
L_{\mathfrak v}
:=
\left\{
 l\subseteq\mathfrak v
 \ \middle|\
 \bigcap_{p\in l}p\neq\varnothing
\right\}.
\]
The members of $L_{\mathfrak v}$ are statements. By convention, the intersection of the empty family is $\Phi$, so $\varnothing\in L_{\mathfrak v}$.
\end{definition}

\Interp A vocabulary is the finite set of distinctions a body can implement at the chosen layer. A statement is a bundle of distinctions that can hold together. The satisfiability condition removes combinations the body cannot instantiate.

\begin{definition}[Extension and freedom]
For $l\in L_{\mathfrak v}$, define
\[
\Ext{l}
:=
\{y\in L_{\mathfrak v}\mid l\subseteq y\}.
\]
The freedom of $l$ is
\[
w(l):=|\Ext{l}|.
\]
\end{definition}

\Interp A completion adds commitments while retaining every commitment in $l$. Freedom counts how many such refinements remain possible. It counts embodied completions rather than parameter settings or a version space.

\begin{definition}[Task and correct policy]
For $X\subseteq L_{\mathfrak v}$, write $\Ext{X}:=\bigcup_{x\in X}\Ext{x}$. A task is a pair $\alpha=\langle I_\alpha,O_\alpha\rangle$ where $I_\alpha\subseteq L_{\mathfrak v}$ and $O_\alpha\subseteq\Ext{I_\alpha}$. A policy is any statement $\pi\in L_{\mathfrak v}$. It is correct for $\alpha$ exactly when
\[
\Ext{I_\alpha}\cap\Ext{\pi}=O_\alpha.
\]
Write $\Pi_\alpha$ for the set of correct policies.
\end{definition}

\Interp Correctness fixes the policy set to the observed task. Freedom compares the completion volume retained by different policies after that requirement has been met. A weakest correct policy is one that maximises freedom within the bounds of correctness. It maintains the integrity of a system without imposing unnecessary constraints.

For continuous languages, Stack Theory replaces counting by a measure $\mu$ and defines $w_\mu(\pi):=\mu(\Ext{\pi})$ \citep{bennett2026h,bennett2025thesis}. The finite theory below uses counting measure. The general best-of-$K$ identity in Section~\ref{sec:general} does not require finiteness.

\subsection{A generative language}

The map from a finite support profile to freedom is perfected\footnote{Well, it becomes ``exact'' but I am loathe to use that word after the LLMs ruined it for me.} after the language states which outputs a generator permits.

\begin{definition}[Permission profiles]
Let $C$ be a finite nonempty set of contexts and let $Y$ be a finite output set with $|Y|\ge 2$. A permission profile is a map
\[
F:C\longrightarrow 2^Y\setminus\{\varnothing\}.
\]
The value $F(c)$ is the nonempty set of outputs permitted at context $c$.
\end{definition}

\Interp A context can be a prompt, a corrupted image, or a masked sequence. A profile records every output the deployed generator can emit with positive probability at each context.

\begin{definition}[Generative permission language]
Let $\Phi$ be the set of all permission profiles. For every $(c,y)\in C\times Y$, define the exclusion program
\[
e_{c,y}:=\{G\in\Phi\mid y\notin G(c)\}.
\]
Let
\[
\mathfrak v:=\{e_{c,y}\mid c\in C,\ y\in Y\}.
\]
For a profile $F$, define its policy
\[
\pi_F:=\{e_{c,y}\in\mathfrak v\mid y\notin F(c)\}.
\]
\end{definition}

\Interp The policy states every exclusion the generator enforces. A generator that permits more outputs asserts fewer exclusions. It therefore embodies a weaker constraint.

\begin{proposition}[Closed-form freedom]\label{prop:freedom}
The map $F\mapsto\pi_F$ is a bijection from permission profiles to $L_{\mathfrak v}$. If $a_c:=|F(c)|$, then
\[
w(\pi_F)
=
\prod_{c\in C}\left(2^{a_c}-1\right).
\]
\end{proposition}

\begin{proof}
A statement $l\subseteq\mathfrak v$ excludes the set
\[
X_c(l):=\{y\in Y\mid e_{c,y}\in l\}
\]
at each context. The programs in $l$ have nonempty intersection exactly when $Y\setminus X_c(l)$ is nonempty for every $c$. Hence statements correspond bijectively to profiles through $F(c)=Y\setminus X_c(l)$.

A statement extends $\pi_F$ exactly when it adds exclusions. Its corresponding profile $G$ therefore obeys
\[
\varnothing\neq G(c)\subseteq F(c)
\]
for every $c$. There are $2^{a_c}-1$ nonempty subsets of $F(c)$, independently across contexts. Go forth and multiply to get the result.
\end{proof}

\Interp Freedom counts every way the generator can commit further without contradicting its present permissions. A context with one permitted output has only one possible refinement left unrefined. A context with three permitted outputs has seven. The product records joint refinement across contexts. ``Completions'' as I so inarticulately put it in earlier work.

A stochastic generator induces conditional laws $Q_c$ on $Y$. I identify its permission profile with
\[
F_Q(c):=\supp Q_c
=
\{y\in Y\mid Q_c(y)>0\}.
\]
Its Stack Theory policy is $\pi_{F_Q}$. Probability mass determines how often exploration reaches each permitted output. Support determines the policy statement. In Proposition~\ref{prop:freedom} I convert the support cardinalities into an extension count.

\begin{table}[t]
\centering
\small
\begin{tabular}{p{2.3cm}p{4.2cm}p{6.2cm}}
\toprule
Object & Definition & Role \\
\midrule
Exploration $K$ & Number of candidate draws inspected in one update & Sampling budget applied during training \\
Generative expressivity $E$ & Greatest mode count retained by an objective minimiser over a data class & Capacity of a training objective \citep{gladstone2026explorative} \\
Freedom $w(\pi)$ & Cardinality or measure of the extension of one trained policy & Compatible completion volume in an embodied language \\
\bottomrule
\end{tabular}
\caption{Three quantities that answer different questions.}
\label{tab:three}
\end{table}

\section{\texorpdfstring{Best-of-$K$ exploration}{Best-of-K exploration}}\label{sec:general}

I now model hard best-of-$K$ training in a form that includes discrete and continuous outputs. I distinguish the law sampled during training from the law used at deployment. This distinction is immaterial for end-to-end XMs, but essential for target-dependent coupling searches.

\begin{definition}[Stochastic generative setting]
Let $(\mathcal C,\mathcal G)$ be a measurable context space and $(\mathcal Y,\mathcal A)$ a measurable output space. Draw a context $c$ from a probability law $\nu$. Let $P$ be a Markov kernel from $\mathcal C$ to $\mathcal Y$, and draw the target $Y$ from $P_c$. Let $\widetilde Q_\pi$ be a Markov kernel from $\mathcal C\times\mathcal Y$ to $\mathcal Y$. Conditional on $(c,Y=y)$, draw an infinite sequence of training candidates $\widehat Y_1,\widehat Y_2,\ldots$ independently from $\widetilde Q_{\pi,c,y}$. Let $Q_\pi$ be the model's deployed Markov kernel from $\mathcal C$ to $\mathcal Y$. I call the exploration \emph{deployment-consistent} when
\[
\widetilde Q_{\pi,c,y}=Q_{\pi,c}
\]
for $\nu(\dd c)P_c(\dd y)$-almost every $(c,y)$. Let
\[
J:\mathcal Y\times\mathcal Y\longrightarrow[0,\infty)
\]
be a measurable loss. The hard best-of-$K$ risk is
\[
R_K(\pi)
:=
\E\left[
\min_{1\le i\le K}J(\widehat Y_i,Y)
\right].
\]
For $t\ge 0$, define the acceptable neighbourhood
\[
A_t(y):=\{\widehat y\in\mathcal Y\mid J(\widehat y,y)\le t\}
\]
and its generator mass
\[
\widetilde q_\pi(c,y,t):=\widetilde Q_{\pi,c,y}(A_t(y)).
\]
\end{definition}

\Interp The loss neighbourhood contains every generation that is good enough at tolerance $t$. The quantity $\widetilde q_\pi(c,y,t)$ is the chance that one training candidate lands there. Exploration repeats that attempt $K$ times. Under deployment consistency, the training candidates are draws from the law used to generate at inference. Otherwise the identity below still holds, but it describes a training-time coupling search rather than sampled access to the deployed policy.

\begin{theorem}[Coverage identity]\label{thm:coverage}
Assume $\E[J(\widehat Y_1,Y)]<\infty$. Then
\[
\boxed{
R_K(\pi)
=
\int_0^\infty
\E_{c,Y}
\left[
\left(1-\widetilde q_\pi(c,Y,t)\right)^K
\right]
\dd t
}.
\]
\end{theorem}

\begin{proof}
Fix $c$, $y$, and $t$. The event
\[
\min_{1\le i\le K}J(\widehat Y_i,y)>t
\]
occurs exactly when every candidate lies outside $A_t(y)$. Independence evaluates that conditional probability as
\[
\left(1-\widetilde Q_{\pi,c,y}(A_t(y))\right)^K.
\]
For any nonnegative integrable random variable $Z$,
\[
\E[Z]=\int_0^\infty\Prb(Z>t)\dd t.
\]
Apply this identity to $Z=\min_iJ(\widehat Y_i,Y)$ and integrate over $c$ and $Y$.
\end{proof}

\begin{remark}[Coupled candidates]
Independence is the only step that produces the power $K$. If the candidates have a joint conditional law $\mathcal Q^{(K)}_{\pi,c,y}$ given $(c,y)$, then
\[
R_K(\pi)
=
\int_0^\infty
\E_{c,Y}
\left[
\mathcal Q^{(K)}_{\pi,c,Y}
\left(
(\mathcal Y\setminus A_t(Y))^K
\right)
\right]
\dd t.
\]
The i.i.d. identity follows when the joint law is $\widetilde Q_{\pi,c,Y}^{\otimes K}$.
\end{remark}

\Interp Candidate diversity changes the joint miss probability even when every marginal law is fixed. The i.i.d. model isolates repeated sampling. A coupled design can lower the joint miss probability without changing its one-candidate marginals.

\Interp Exploration applies the miss function $(1-q)^K$ to every acceptable region. Equivalently, the chance of at least one hit is
\[
g_K(q):=1-(1-q)^K.
\]
For small $q$,
\[
g_K(q)=Kq+O(q^2).
\]
At a fixed candidate law, exploration amplifies small positive mass. A compatible region with zero mass remains inaccessible at every finite $K$. Exploration is leverage on rare permissions rather than a source of new ones.

\begin{corollary}[Marginal value of another candidate]\label{cor:marginal}
Under the assumptions of Theorem~\ref{thm:coverage},
\[
R_K(\pi)-R_{K+1}(\pi)
=
\int_0^\infty
\E_{c,Y}
\left[
\widetilde q_\pi(c,Y,t)
\left(1-\widetilde q_\pi(c,Y,t)\right)^K
\right]
\dd t
\ge 0.
\]
\end{corollary}

\begin{proof}
Subtract the identity for $K+1$ from the identity for $K$ and factor the integrand.
\end{proof}

\Interp Another candidate helps where acceptable mass is positive yet often missed. For a fixed $K$, the factor $q(1-q)^K$ is largest at $q=1/(K+1)$. Regions that are almost certain need little extra exploration. Regions with zero mass cannot be reached.

\begin{theorem}[Essential-infimum limit]\label{thm:support}
Define
\[
\widetilde d_\pi(c,y)
:=
\essinf_{\widehat y\sim \widetilde Q_{\pi,c,y}}J(\widehat y,y).
\]
Assume $\widetilde d_\pi$ is measurable. If $\E[J(\widehat Y_1,Y)]<\infty$, then
\[
\min_{1\le i\le K}J(\widehat Y_i,Y)
\longrightarrow
\widetilde d_\pi(c,Y)
\]
almost surely, and
\[
\boxed{
R_K(\pi)
\longrightarrow
\E_{c,Y}[\widetilde d_\pi(c,Y)]
}
\]
as $K\to\infty$.
\end{theorem}

\begin{proof}
For fixed $c$ and $y$, each sampled loss is at least $\widetilde d_\pi(c,y)$ outside a $\widetilde Q_{\pi,c,y}$-null set. A countable union over the sampled sequence is still null, so the sequence of minima is almost surely nonincreasing and bounded below by $\widetilde d_\pi(c,y)$. For every $n\ge1$, the set
\[
\left\{\widehat y\mid J(\widehat y,y)<\widetilde d_\pi(c,y)+\frac1n\right\}
\]
has positive $\widetilde Q_{\pi,c,y}$ mass by the definition of essential infimum. The probability that the infinite i.i.d. sequence never visits this set is zero. Intersecting these probability-one events over $n\in\mathbb N$ shows that the decreasing sequence of minima converges to $\widetilde d_\pi(c,y)$ almost surely. Integrating the conditional probability-one statement over $(c,Y)$ extends this to joint almost-sure convergence. Every minimum is bounded above by $J(\widehat Y_1,Y)$, so dominated convergence yields the risk limit.
\end{proof}

\Interp Infinite hard exploration retains only the best loss attainable under the training candidate law. If $\mathcal Y$ is a Polish space\footnote{Not invaded.}, $Q_{\pi,c}$ is a Borel probability law, exploration is deployment-consistent, and $J(\widehat y,y)$ is continuous in $\widehat y$, then $\widetilde d_\pi(c,y)=\inf_{\widehat y\in\supp Q_{\pi,c}}J(\widehat y,y)$. If $J(\widehat y,y)=d(\widehat y,y)$ for the metric $d$, this is the distance from $y$ to deployed support. Under those conditions the theorem recovers the support-coverage limit identified by Gladstone, Ji, and Du \citep{gladstone2026explorative}. Yay.

In Sections~\ref{sec:finite}--\ref{sec:nonuniform}, I assume deployment consistency. I therefore write $Q_{\pi,c}$ for both the training and deployed law.

\section{Finite embodied outputs}\label{sec:finite}

The general identity is ``embodied'' and interpreted as a count of completions (degrees of freedom), when the embodied vocabulary distinguishes finitely many outputs.

\begin{definition}[Identification loss]
Let $V(c):=\{y\in Y\mid P_c(Y=y)>0\}$ be the target support at context $c$. I write
\[
p_c(y):=P_c(Y=y),
\qquad
q_c(y):=Q_{\pi,c}(\widehat Y=y).
\]
The identification loss is
\[
J(\widehat y,y):=\one[\widehat y\neq y].
\]
\end{definition}

\begin{definition}[Soundness, completeness, and correctness]
At context $c$, a generator is sound when $\supp Q_{\pi,c}\subseteq V(c)$, complete when $V(c)\subseteq\supp Q_{\pi,c}$, and correct when $\supp Q_{\pi,c}=V(c)$.
\end{definition}

\Interp Soundness limits what a generator may produce, since every output it can emit must be one that actually occurs at $c$. On the other hand completeness limits what it may omit, since every output that occurs must stay within reach. A generator that always returns the same valid not-hotdog is sound and incomplete, because everything it produces is correct and almost everything correct is missing. A generator that reaches every valid not-hotdog and sometimes emits noise is complete and unsound. Correctness holds when what a generator can produce is exactly what occurs. Soundness on its own is easy, and the collapsed generator has it. What separates that generator from a correct one is how much of the valid set stays available, which is what freedom counts.

\begin{theorem}[Exact finite risk]\label{thm:finite}
For identification loss,
\[
\boxed{
R_K(\pi)
=
\E_{c\sim\nu}
\left[
\sum_{y\in V(c)}p_c(y)\left(1-q_c(y)\right)^K
\right]
}.
\]
The hit probability is
\[
H_K(\pi):=1-R_K(\pi)
=
\E_{c\sim\nu}
\left[
\sum_{y\in V(c)}p_c(y)
\left(1-\left(1-q_c(y)\right)^K\right)
\right].
\]
\end{theorem}

\begin{proof}
A target $y$ is missed exactly when all $K$ candidates differ from $y$. One candidate differs from $y$ with probability $1-q_c(y)$, and the candidates are independent, so all $K$ differ with probability $(1-q_c(y))^K$. Weighting that by $p_c(y)$ and taking the expectation over $c\sim\nu$ produces the displayed
risk. The hit probability is its complement.
\end{proof}

\Interp Three quantities determine the risk. First the target frequency $p_c(y)$ is how often each valid output is asked for, second the generator mass $q_c(y)$ is how often
the model produces it, and third $K$ is how many attempts each comparison allows. As $K$ grows, $(1-q_c(y))^K$ falls to zero for every output with positive mass, so in the limit only the support remains. At finite $K$, the size of that mass also affects the risk. A model can permit an output and still produce it so rarely
that $(1-q_c(y))^K$ stays near one for a long time. A count over the support therefore cannot rank finite-$K$ performance on its own.

\begin{corollary}[Expected distinct coverage]\label{cor:distinct}
Fix one context with $m:=|V(c)|$ and a uniform target distribution on $V(c)$. Let $N_K$ be the number of distinct valid output cells observed among the $K$ candidate samples. Then
\[
\boxed{
H_K(\pi\mid c)
=
\frac{\E[N_K]}{m}
}.
\]
\end{corollary}

\begin{proof}
Each valid cell $y$ appears among the candidates with probability $1-(1-q_c(y))^K$. Therefore
\[
\E[N_K]
=
\sum_{y\in V(c)}
\left(1-(1-q_c(y))^K\right).
\]
Divide by $m$ and apply Theorem~\ref{thm:finite}.
\end{proof}

\Interp Best-of-$K$ success equals normalised distinct-output coverage. The hit probability is the expected fraction of valid output cells reached by the samples. The permission language then maps the deployed support count to the local freedom factor $2^a-1$. Latent draws count only after the decoder converts them into output distinctions.

\begin{proposition}[Output counting and recoding]\label{prop:recoding}
Let $Z$ be a latent space and $g:Z\to Y$ a decoder. For latent candidates $z_1,\ldots,z_K$,
\[
|\{g(z_1),\ldots,g(z_K)\}|
\le
|\{z_1,\ldots,z_K\}|
\le K.
\]
Every latent or parameter recoding that leaves each conditional output law $Q_c$ unchanged also leaves $F_Q$, $\pi_{F_Q}$, and $w(\pi_{F_Q})$ unchanged.
\end{proposition}

\begin{proof}
The image of a finite set under a function cannot contain more elements than the set itself. Regarding the second claim, $F_Q(c)$ is the support of $Q_c$, so any recoding that leaves every $Q_c$ unchanged leaves every $F_Q(c)$ unchanged. The permission-profile bijection then fixes $\pi_{F_Q}$, and Proposition~\ref{prop:freedom} computes $w(\pi_{F_Q})$ from the support cardinalities by themselves. None of the three quantities refer to a latent or a parameter, so none of them can move when the output law is held fixed.
\end{proof}

\Interp Several latent values can decode to the same output. Counting latents can therefore overstate the distinctions the model can express at the chosen output layer. Freedom follows function rather than parameter form. Hence the third axis claim. 

\section{When exploration ranks freedom}

I first hold target frequencies exchangeable. This isolates the relation between sampled coverage and the support of the policy.

\begin{theorem}[Uniform targets]\label{thm:uniform}
Fix one context with $m$ valid outputs. Let the target distribution be uniform on $V$ and let $q$ range over all probability distributions on $Y$.
\begin{enumerate}[leftmargin=*,label=\arabic*.]
    \item For $K=1$, every distribution supported on $V$ minimises $R_1$.
    \item For every $K\ge 2$, the unique minimiser is the uniform distribution on $V$.
\end{enumerate}
Consequently, for $K\ge 2$, the optimum has support $V$ and maximises the local freedom factor $2^{|\supp q|}-1$ among sound policies.
\end{theorem}

\begin{proof}
Any mass outside $V$ can be moved to a valid cell and weakly reduce every term in the risk. Hence an optimum is supported on $V$. For $K=1$,
\[
R_1(q)
=
1-\frac{1}{m}\sum_{y\in V}q(y)
=
1-\frac{1}{m}
\]
for every sound $q$.

For $K\ge2$, the map $x\mapsto(1-x)^K$ is strictly convex on $[0,1]$. Jensen's inequality then produces the chain
\[
\frac1m\sum_{y\in V}(1-q(y))^K
\ge
\left(1-\frac1m\sum_{y\in V}q(y)\right)^K
=
\left(1-\frac1m\right)^K.
\]
Equality holds exactly when $q(y)=1/m$ for every $y\in V$.
\end{proof}

\Interp One sampled output cannot distinguish a collapsed generator from a covering one, so at $K=1$ freedom is not visible to training. However with second or later sample it is, and from that point on the unique population optimum is the weakest sound permission profile. Exploration is the enabler of that visibility, and $K=2$ is where we can start to see its effect\footnote{Thereby further validating years of my life spent in relative obscurity and poverty.}.

The next theorem holds support size fixed and isolates the effect of balanced mass.

\begin{theorem}[Balanced support]\label{thm:balanced}
Fix $m$ uniformly likely valid outputs. Let a sound generator place uniform mass on a permitted subset of size $a$, where $1\le a\le m$. Its hit probability is
\[
\boxed{
H_K^*(a)
=
\frac{a}{m}
\left[
1-\left(1-\frac1a\right)^K
\right]
}.
\]
For $K=1$, $H_1^*(a)=1/m$ for every $a$. For every $K\ge2$, $H_K^*(a)$ is strictly increasing in $a$ and therefore strictly increasing in the local freedom factor $2^a-1$.
\end{theorem}

\begin{proof}
A target outside the permitted subset is always missed. A target inside it is hit with probability $1-(1-1/a)^K$. Averaging over the $m$ uniform targets resulting in the formula.

For $K\ge2$, first compare $a=1$ and $a=2$. We have $H_K^*(1)=1/m$ and
\[
H_K^*(2)=\frac{2}{m}\left(1-2^{-K}\right)>\frac1m.
\]
For $a>1$, extend $a$ to a continuous variable and define
\[
h_K(a):=a\left[1-\left(1-\frac1a\right)^K\right].
\]
Its derivative is
\[
h_K'(a)
=
1-
\left(1-\frac1a\right)^{K-1}
\left(1+\frac{K-1}{a}\right).
\]
Let $x=1/a\in(0,1)$ and $n=K-1$. The logarithm of the product subtracted above is
\[
f(x):=n\log(1-x)+\log(1+nx).
\]
For $x\in(0,1)$,
\[
f'(x)
=
-\frac{n(n+1)x}{(1-x)(1+nx)}<0,
\]
while $f(0)=0$. Hence the product is strictly below one and $h_K'(a)>0$ for $a>1$. Together with the direct $a=1$ comparison, this proves strict increase on the integers.
\end{proof}

\Interp Under the symmetry assumptions, the objective's preference order over balanced sound policies is exactly the freedom order. At every $K$ above one, training pays weaker policies strictly more, and Figure~\ref{fig:balanced} is the payment schedule.

\begin{theorem}[Complementarity]\label{thm:complement}
Under the assumptions of Theorem~\ref{thm:balanced}, the gain from one additional candidate is
\[
\boxed{
H_{K+1}^*(a)-H_K^*(a)
=
\frac1m\left(1-\frac1a\right)^K
}.
\]
For every fixed $K\ge1$, this gain is strictly increasing in $a$.
\end{theorem}

\begin{proof}
Subtract the two expressions from Theorem~\ref{thm:balanced}, and the remaining term is $(1/m)(1-1/a)^K$, which rises strictly with $a$.
\end{proof}

\Interp Freedom and exploration are complements under the theorem's symmetry assumptions. A policy that permits one output gains nothing from repeated sampling. A policy with many balanced permitted outputs continues to reveal new distinctions. For fixed target-support size, local freedom increases the marginal value of exploration. Exploration then increases the degree to which training distinguishes local freedom.

\begin{proposition}[Entropy balances access]\label{prop:entropy}
Fix a sound support $A\subseteq V$ of size $a$ under uniform targets. For every $K\ge2$, the uniform law on $A$ uniquely maximises the best-of-$K$ hit probability among laws supported on $A$. It also uniquely maximises Shannon entropy on $A$. If the support may vary among sound laws, the uniform law on all of $V$ uniquely maximises both hit probability and entropy, and its permission profile has maximal local freedom among sound profiles.
\end{proposition}

\begin{proof}
On a fixed support $A$, the hit probability is
\[
\frac1m\sum_{y\in A}\left[1-(1-q(y))^K\right].
\]
The bracketed function is strictly concave for $K\ge2$, so Jensen's inequality gives a unique maximum at $q(y)=1/a$. Shannon entropy has the same unique maximiser on a fixed finite support. Theorem~\ref{thm:uniform} establishes the full-support claim, and Proposition~\ref{prop:freedom} maximal local freedom.
\end{proof}

\Interp Entropy spreads mass across a permitted output set, while freedom is the count of those completions or refinements of the corresponding policy statement still left to pursue. Degrees of freedom remaining, so to speak. Freedom is how loose your straight jacket feels. Entropy and freedom select the same law in this particular symmetric finite setting, but they remain distinctly different quantities. In a structured embodied language, supports with equal entropy or equal cardinality are very likely to have different extension structure.

\begin{figure}[t]
\centering
\includegraphics[width=0.78\linewidth]{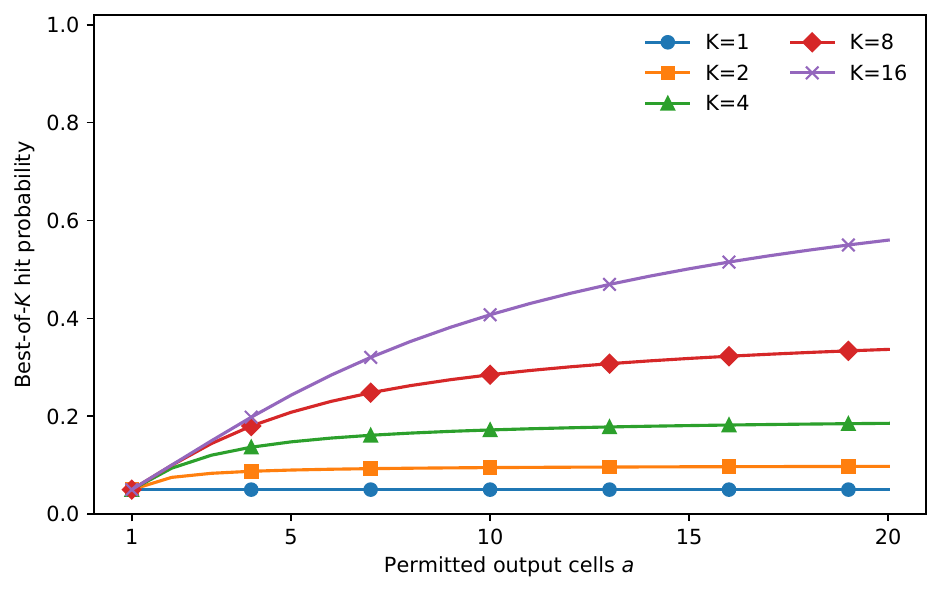}
\caption{Best-of-$K$ hit probability for a uniform target over $m=20$ valid output cells when the generator is uniform on $a$ permitted cells. For $K=1$ the curve is flat. For every $K>1$ sampled coverage rises with the local freedom factor $2^a-1$.}
\label{fig:balanced}
\end{figure}

\begin{example}[Freedom alone does not rank finite-$K$ access]\label{ex:skew}
Let $m=3$ and $K=2$. A generator with
\[
q^{(1)}=(1/2,1/2,0)
\]
has support size two and hit probability
\[
H_2(q^{(1)})=\frac12.
\]
A generator with
\[
q^{(2)}=(0.998,0.001,0.001)
\]
has support size three and hit probability
\[
H_2(q^{(2)})\approx0.3347.
\]
The second policy has greater freedom in the permission language, yet lower finite-$K$ access.
\end{example}

\Interp Probability mass is the intermediary between permission and access. A generator may hold a commodious permitted set and yet visit scarcely any of it, should almost the whole of its mass repose upon a single output. The second
generator above is the weaker of the two in the permission language, its
extension factor standing at seven against three. It is the poorer
performer at $K=2$ notwithstanding. Finite exploration adjudicates what a
generator reaches rather than what it is at liberty to produce. Freedom
resumes its authority only as $K$ grows without bound, when even the most meagre allocation is visited in the end.

\section{Nonuniform targets}\label{sec:nonuniform}

The next result shows how finite exploration balances target frequency against support coverage.

\begin{theorem}[Optimal mass under nonuniform targets]\label{thm:kkt}
Fix one context with $m\ge2$ valid outputs and strictly positive target probabilities $p_1,\ldots,p_m$. For identification loss, let
\[
R_K(q)=\sum_{j=1}^m p_j(1-q_j)^K
\]
over the probability simplex.
\begin{enumerate}[leftmargin=*,label=\arabic*.]
    \item For $K=1$, every optimum is supported on the outputs with maximal target probability.
    \item For $K>1$, the optimum is unique and has coordinates
    \[
    \boxed{
    q_j^*
    =
    \left[
    1-
    \left(\frac{\lambda}{Kp_j}\right)^{1/(K-1)}
    \right]_+
    }
    \]
    where $[x]_+:=\max(x,0)$ and $\lambda$ is chosen so that $\sum_jq_j^*=1$.
\end{enumerate}
\end{theorem}

\begin{proof}
For $K=1$,
\[
R_1(q)=1-\sum_{j=1}^mp_jq_j,
\]
so the optimum concentrates on the largest $p_j$ values.

For $K>1$, $R_K$ is strictly convex. Introduce a multiplier $\lambda$ for $\sum_jq_j=1$ and multipliers $\mu_j\ge0$ for $q_j\ge0$. The stationarity\footnote{Is this a word? Well, it is now.} condition is
\[
-Kp_j(1-q_j)^{K-1}+\lambda-\mu_j=0.
\]
For active coordinates $q_j>0$, complementary slackness results in $\mu_j=0$ and the displayed formula. For inactive coordinates, $q_j=0$ when $\lambda\ge Kp_j$. Strict convexity rules out any second minimiser.
\end{proof}

\Interp At finite $K$ the optimum is parsimonious. It bestows mass only upon targets whose frequency commends them, since $q_j^*$ is positive exactly when $Kp_j$ exceeds $\lambda$, and among those so admitted the more frequent are the
more generously endowed. Exploration erodes this parsimony. Further mass upon a target the candidates reach many times over purchases little, whereas the first mass conferred upon a neglected target purchases a great deal, and the disparity widens as the candidates multiply. The admitting threshold
$\lambda/K$ falls in consequence, and at length lies beneath every positive
target frequency. The optima then approach uniform mass across the whole of the target support, and an optimum once governed by frequency comes to be governed by coverage.

\begin{corollary}[Full-support limit]\label{cor:uniformlimit}
Under the assumptions of Theorem~\ref{thm:kkt}, every output is active for all sufficiently large $K$, and
\[
q_j^*\longrightarrow\frac1m
\]
for every $j$ as $K\to\infty$.
\end{corollary}

\begin{proof}
Let $A_K:=\{j:q_j^*>0\}$ and write $r_K:=|A_K|$. The case $r_K=1$ cannot satisfy the Karush--Kuhn--Tucker conditions for $K>1$, because the active coordinate would force it to be the case that $\lambda=0$, while every inactive coordinate requires $\lambda\ge Kp_j>0$. Hence $r_K\ge2$.

For any fixed active set $A$ of size $r\ge2$, the active-coordinate formula rearranges to
\[
\left(\frac{\lambda}{K}\right)^{1/(K-1)}
=
\frac{r-1}{\sum_{i\in A}p_i^{-1/(K-1)}}.
\]
If some $j\notin A$ stayed inactive along an unbounded sequence of $K$, then the right-hand side would tend to $(r-1)/r<1$. Therefore $\lambda/K$ would tend to zero exponentially, while inactivity requires $\lambda/K\ge p_j>0$. This is impossible. Since there are finitely many supports, every coordinate is active for all sufficiently large $K$.

With full support, write
\[
a_K:=\left(\frac{\lambda}{K}\right)^{1/(K-1)}
=
\frac{m-1}{\sum_{i=1}^mp_i^{-1/(K-1)}}.
\]
Then
\[
q_j^*=1-a_Kp_j^{-1/(K-1)}.
\]
Every factor $p_j^{-1/(K-1)}$ tends to one, so $a_K\to(m-1)/m$ and $q_j^*\to1/m$.
\end{proof}

\Interp Finite exploration trades frequency for support breadth, which determines the local freedom factor in the declared permission language. As $K$ grows, the unique finite-$K$ population optima approach uniform sampling across the positive target support. The limiting hard objective itself retains only support coverage and does not select a unique density.

\begin{figure}[t]
\centering
\includegraphics[width=0.78\linewidth]{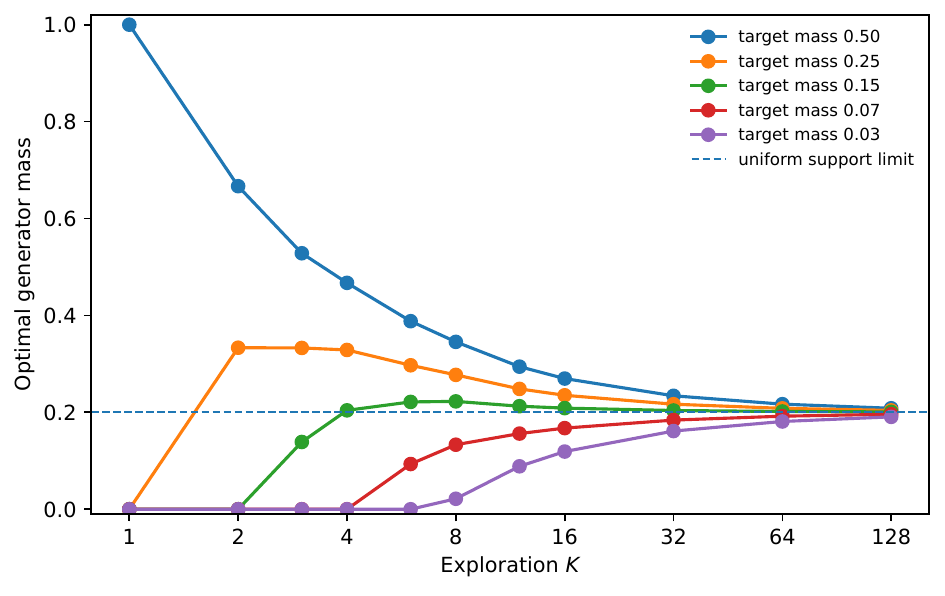}
\caption{The exact optimum from Theorem~\ref{thm:kkt} for target probabilities $(0.50,0.25,0.15,0.07,0.03)$. At $K=1$ the optimum collapses onto the most frequent target. As $K$ grows, all target cells receive mass and the optimum approaches the uniform support limit.}
\label{fig:nonuniform}
\end{figure}

\section{Future compatibility and generalisation}

Exploration concerns sampled access to permitted outputs on scored contexts. Freedom concerns refinements of the whole policy, including unseen contexts. I now connect the two to future requirements.

\begin{definition}[Observed and unseen contexts]
Let $C_\alpha\subsetneq C$ be the observed contexts and let $C_u:=C\setminus C_\alpha$. Let $V_\alpha:C_\alpha\to2^Y\setminus\{\varnothing\}$ record the observed valid outputs. A profile $F$ is child-correct when
\[
F(c)=V_\alpha(c)
\]
for every $c\in C_\alpha$.
\end{definition}

\Interp Child correctness fixes the permission profile where targets have been observed. Generalisation depends on what the policy leaves open elsewhere. Freedom is your remaining freedom to move into whatever requirements the unforeseen stipulates.

\begin{proposition}[Canonical Stack correctness]\label{prop:stackcorrect}
For each observed context $c\in C_\alpha$, let
\[
\mathfrak v_c:=\{e_{c,y}\mid y\in Y\}
\]
be the local vocabulary and let $\pi_A^c:=\{e_{c,y}\mid y\notin A\}$ for every nonempty $A\subseteq Y$. Write $\Extc{l}:=\{l'\in L_{\mathfrak v_c}\mid l\subseteq l'\}$ for local statement extension and $\Extc{X}:=\bigcup_{l\in X}\Extc{l}$ for sets of local statements. Define the local Stack task
\[
\alpha_c
:=
\left\langle
\{\varnothing\},
\Extc{\pi_{V_\alpha(c)}^c}
\right\rangle.
\]
Then $\pi_A^c$ is correct for $\alpha_c$ exactly when $A=V_\alpha(c)$. Consequently, a profile $F$ is child-correct exactly when $\pi_{F(c)}^c\in\Pi_{\alpha_c}$ for every observed context.
\end{proposition}

\begin{proof}
The empty statement has every local statement as a completion, so
\[
\Extc{\{\varnothing\}}=L_{\mathfrak v_c}.
\]
Therefore $\pi_A^c$ is correct exactly when
\[
\Extc{\pi_A^c}=\Extc{\pi_{V_\alpha(c)}^c}.
\]
Equality of extensions implies equality of the statements. Each statement contains itself as a completion, so equality implies both $\pi_A^c\subseteq\pi_{V_\alpha(c)}^c$ and the reverse inclusion. The permission-profile bijection then translates this into $A=V_\alpha(c)$.
\end{proof}

\Interp The profile equality above is the canonical Stack Theory correctness predicate applied context by context. No new notion of correctness is needed.

\begin{definition}[Independent nonempty demands]
For every unseen context $c\in C_u$, draw a future demand
\[
S_c
\sim
\operatorname{Unif}\left(2^Y\setminus\{\varnothing\}\right)
\]
independently across contexts. A profile $F$ is compatible with the demands when
\[
S_c\subseteq F(c)
\]
for every $c\in C_u$.
\end{definition}

\Interp A future task adds one or more positive requirements at each new context. The profile remains compatible when it has not excluded any required output. This demand law tests compatibility with future requirements. It does not declare every permitted output correct, and it does not penalise extra unseen outputs.

\begin{theorem}[Future compatibility is proportional to freedom]\label{thm:compat}
Let $a_c:=|F(c)|$. Under independent nonempty demands,
\[
\boxed{
\Prb(F\text{ is compatible})
=
\prod_{c\in C_u}
\frac{2^{a_c}-1}{2^{|Y|}-1}
}.
\]
Define the unseen freedom factor
\[
w_u(F):=\prod_{c\in C_u}(2^{a_c}-1).
\]
Then
\[
\Prb(F\text{ is compatible})
=
\frac{w_u(F)}{(2^{|Y|}-1)^{|C_u|}}.
\]
Among child-correct profiles,
\[
\boxed{
\Prb(F\text{ is compatible})
=
\frac{w(\pi_F)}
{\left[\prod_{c\in C_\alpha}(2^{|V_\alpha(c)|}-1)\right]
(2^{|Y|}-1)^{|C_u|}}
}.
\]
Hence future compatibility probability and total freedom induce exactly the same ordering on child-correct profiles.
\end{theorem}

\begin{proof}
At context $c$, there are $2^{|Y|}-1$ possible nonempty demands. Exactly $2^{a_c}-1$ of them are nonempty subsets of $F(c)$. Independence across contexts multiplies these per-context ratios into the first formula.

For a child-correct profile, Proposition~\ref{prop:freedom} expands the total freedom as
\[
w(\pi_F)
=
\left[\prod_{c\in C_\alpha}(2^{|V_\alpha(c)|}-1)\right]
\left[\prod_{c\in C_u}(2^{a_c}-1)\right].
\]
Substitute the unseen product from the first formula.
\end{proof}

\Interp A pooled demand law over all context-output pairs prices commitments additively and can reverse the total freedom ordering. Independent nonempty context demands recover the product structure of the embodied language, and under this positive-demand law freedom determines future compatibility exactly, up to a constant no policy can change. The theorem concerns future compatibility. Exact unseen-output correctness would require a law that also states which outputs are forbidden.

The measurable Stack Theory result is broader. Let $U$ be the unseen region of a task and let
\[
B_\pi:=\Ext{\pi}\cap U
\]
be the unseen buffer of a correct policy. Under a $\mu$-exchangeable demand law, generalisation probability is a nondecreasing function of $\mu(B_\pi)$, and strict under nondegeneracy \citep{bennett2026h,bennett2025thesis}. Correctness fixes the observed part of the extension, so ranking by buffer measure agrees with ranking by $\mu$-freedom.

In the finite permission model, under deployment consistency and the independent-demand assumptions, the full relation is
\[
K
\longrightarrow
1-(1-q)^K
\longrightarrow
Q_{\pi,c}
\longrightarrow
F_Q
\longrightarrow
\pi_{F_Q}
\longrightarrow
\Ext{\pi_{F_Q}}
\longrightarrow
\Prb(F_Q\text{ remains compatible}).
\]

\Interp The first arrow is the exact sampling transform under deployment consistency. Training and optimisation connect that transform to the learned law $Q_{\pi,c}$. The permission-language construction converts the support profile of that law into $\Ext{\pi_{F_Q}}$. The final arrow is Theorem~\ref{thm:compat}. Under the broader Stack Theory extension model, the corresponding final quantity is generalisation probability. The training step is empirical. Increasing $K$ need not enlarge the learned extension in every architecture or dataset. The reported XM experiments are consistent with that possibility but do not measure freedom directly \citep{gladstone2026explorative}.

\begin{corollary}[Exploration accesses freedom]\label{cor:bridge}
Assume a uniform target law at each scored context and a sound generator with balanced mass on its permitted outputs. Assume independent nonempty demands on unseen contexts. Then for every $K>1$
\begin{enumerate}[leftmargin=*,label=\arabic*.]
    \item local best-of-$K$ hit probability is strictly increasing in the permitted-output count $a_c$, and therefore in the local freedom factor $2^{a_c}-1$,
    \item for fixed $m=|V(c)|$, the gain from increasing $K$ is strictly larger when $a_c$ is larger, and
    \item unseen future compatibility probability is exactly proportional to total freedom among child-correct profiles.
\end{enumerate}
\end{corollary}

\begin{proof}
Items 1 and 2 follow from Theorems~\ref{thm:balanced} and \ref{thm:complement}. Item 3 follows from Theorem~\ref{thm:compat}.
\end{proof}

\Interp At a scored context, exploration ranks the local extension factor through repeated samples under the symmetry assumptions. Across unseen contexts, future compatibility multiplies those same factors. Probability allocation mediates the first relation. Freedom fixes the second under the stated demand law. Exploration creates the pressure. Freedom is the quantity it acts on.

\section{Deterministic heads}

The claim that $K$ caps support holds for an enumerated set of $K$ deterministic heads. It does not hold for stochastic Forward XM.

\begin{proposition}[Deterministic special case]\label{prop:heads}
Let $f_1,\ldots,f_K:C\to Y$ be deterministic candidate functions evaluated together. Their committed set at context $c$ is
\[
F_f(c):=\{f_1(c),\ldots,f_K(c)\}.
\]
Then
\[
|F_f(c)|\le K.
\]
Consequently, the local freedom factor in the permission language obeys
\[
2^{|F_f(c)|}-1\le2^K-1.
\]
\end{proposition}

\begin{proof}
A set formed from $K$ values contains at most $K$ elements. Proposition~\ref{prop:freedom} converts that bound into the freedom factor.
\end{proof}

\Interp A fixed bank of $K$ heads has a $K$-output ceiling. Forward XM draws $K$ samples afresh from one conditional law. That law can have support larger than $K$.

\begin{example}[Two modes and two samples]
Let one context have two equiprobable valid outputs under identification loss. Two deterministic heads can place one output on each head and attain zero loss. A stochastic generator with probability $1/2$ on each output and $K=2$ misses the target with probability $1/4$. Its risk is therefore $1/4$ under unit identification loss. With squared distance between distinct one-hot outputs equal to two, the corresponding expected loss is $1/2$.
\end{example}

\Interp Deterministic heads and stochastic draws solve different population problems. The first enumerates candidates. The second samples from a distribution. The general XM relation is the coverage identity, not a support ceiling of $K$.

\section{Generative expressivity}

Gladstone, Ji, and Du define generative expressivity as
\[
E
:=
\sup_{p^*,c}
\sup_{\theta^*\in\arg\min_\theta\mathcal L(\theta)}
M(P_{\theta^*}(\cdot\mid c)),
\]
where $M$ counts modes \citep{gladstone2026explorative}. This is an objective-level capacity. Under their separated-mode assumptions, their smooth Forward XM result establishes $E\ge K$. It does not state that a stochastic generator has at most $K$ modes.

Freedom answers a different question. It is evaluated on a particular trained policy in a particular embodied language. Two models can share the same objective, the same $K$, and the same objective-level $E$ while inducing different support profiles and different freedom. Conversely, two distributions can have the same mode count while their extension structure differs.

The hard infinite-$K$ limit brings the distinction into focus. Theorem~\ref{thm:support} and the XM appendix both show that every support-covering generator becomes optimal in that limit \citep{gladstone2026explorative}. The objective then fixes coverage without fixing density. Freedom can still differ when the embodied language tracks further compatible refinements beyond raw support count.

\Interp Generative expressivity describes what an objective can retain. Exploration describes how many draws training examines. Freedom describes how much functional freedom one trained policy retains. These quantities can correlate. They remain distinct.

\section{Experiments}\label{sec:experiments}

The theory separates two empirical questions. The first asks whether gradient training converts a larger exploration budget into a weaker deployed policy. The second asks whether freedom improves selection among models trained by XM. Experiment 1 tests the first relation. Experiment 2 tests the second. Code, frozen specifications, and result tables accompany this manuscript.

\subsection{Experiment 1. Exploration raises freedom}

I used a latent-conditioned continuous generator with twelve one-hot output prototypes. The full profile fixed input dimension eight, latent dimension twelve, network width 128, 16,384 training cases, 4,096 audit contexts, batch size 384, 2,500 updates, twelve paired repetitions, and
\[
K\in\{1,2,4,8,16,32\}.
\]
At each update, the code drew $K$ Gaussian latent candidates for every target, generated $K$ continuous output vectors, computed mean squared reconstruction loss, retained the lowest-loss candidate, and recomputed only that candidate with gradients. This follows the public Forward XM winner-only update. The analytic categorical loss used in the earlier experiment is absent.

Within each repetition, every value of $K$ used the same training set, target law, initial parameters, minibatch stream, audit contexts, and initial candidate random-number state. The two target laws were uniform and context-dependent. Every target cell was admissible. A generated vector outside every declared cell was counted as invalid rather than as another output.

At deployment, I drew 240 fresh latent samples at each audit context. A generation occupied output cell $j$ when its nearest one-hot prototype lay within radius $0.64$. The active cells defined the deployed permission profile
\[
F_\theta(x):=\{j:\text{cell }j\text{ is occupied by at least one deployment draw}\}.
\]
Writing $a_\theta(x):=|F_\theta(x)|$, the measured local freedom factor was $2^{a_\theta(x)}-1$. I report the audit mean of $\log(2^{a_\theta(x)}-1)$. An empty operational profile lies outside the nonempty permission language, so it received the minimum score zero and its frequency was recorded separately.

The full study completed 144 of 144 jobs. Figure~\ref{fig:eaw1} shows the primary freedom result. Relative to $K=1$, the measured permission profile broadened from $K=4$ onward under uniform targets and from $K=2$ onward under the context-dependent law. The uniform profile reached all twelve cells by $K=16$ and then saturated.

\begin{figure}[t]
\centering
\includegraphics[width=.78\linewidth]{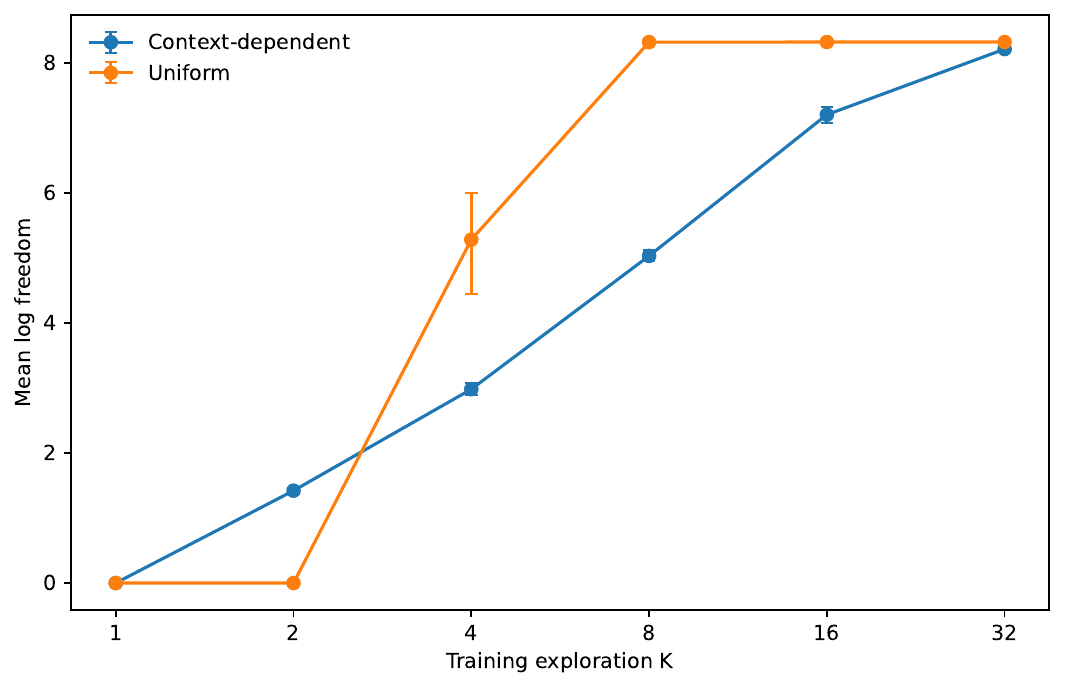}
\caption{Experiment 1. Mean log freedom after sample-based Forward XM training. Points show means across twelve paired repetitions. Error bars are 95 percent bootstrap intervals.}
\label{fig:eaw1}
\end{figure}

\begin{table}[t]
\centering
\small
\begin{tabular}{llrrr}
\toprule
Target law & $K$ & Best-of-8 hit & Active outputs & Mean log freedom \\
\midrule
Uniform & 1  & 0.000000 & 0.000000 & 0.000000 \\
Uniform & 8  & 0.342462 & 11.998739 & 8.316647 \\
Uniform & 32 & 0.440817 & 12.000000 & 8.317522 \\
Context-dependent & 1  & 0.319796 & 0.572815 & 0.000179 \\
Context-dependent & 8  & 0.711196 & 7.274760 & 5.029522 \\
Context-dependent & 32 & 0.648870 & 11.843648 & 8.209101 \\
\bottomrule
\end{tabular}
\caption{Experiment 1 cohort means. The table shows $K=1$, $K=8$, and $K=32$. Figure~\ref{fig:eaw1} includes every tested value.}
\label{tab:eaw1}
\end{table}

Under uniform targets, $K=8$ increased mean log freedom over $K=1$ by $8.316647$, with paired 95 percent interval $[8.316281,8.316958]$. At $K=32$ the increase was $8.317522$. Under context-dependent targets, the corresponding increases were $5.029343$, with interval $[4.953404,5.112625]$, and $8.208922$, with interval $[8.198358,8.220584]$. Every one of the twelve paired differences was positive in these four comparisons. The uniform $K=2$ condition tied $K=1$ in log freedom. From $K=4$ onward every uniform paired difference was positive. Under the context-dependent law every comparison from $K=2$ onward was positive.

Best-of-eight hit and freedom separated under the context-dependent law. Hit rose from $0.319796$ at $K=1$ to $0.711196$ at $K=8$, then fell to $0.648870$ at $K=32$, while active support and freedom continued to rise. This is the probability-mass mediation predicted by the theory. A broader permission profile does not guarantee greater finite-$K$ access when sampling mass becomes less favourable.

The future-demand check drew independent nonempty output demands at each audit context. Observed compatibility differed from the extension-count prediction by $2.06\times10^{-4}$ on average and $2.19\times10^{-3}$ at most.

\Interp Experiment 1 closes the training arrow in this finite operational language. The public XM update changed the deployed permission profile. Freedom measured that change, while the context-dependent condition also showed why freedom and finite-$K$ access remain distinct.

\subsection{Experiment 2. Freedom-selected XM beats child-validation XM}

Experiment 2 tested whether freedom improves selection among models trained by the same sample-based Forward XM procedure. Each world used twelve decoded output cells. I inserted one child example of every output, then sampled the remaining child labels from a randomly permuted Zipf law with exponent $1.6$. The parent extension retained the observed child slice and added new contexts drawn from the same context law, with those twelve outputs weighted uniformly on the added slice. The design therefore changed demanded output frequency rather than introducing outputs absent from training.

Each candidate was trained by winner-only Forward XM. Every update drew $K$ Gaussian latent candidates, generated one continuous output from each latent, measured reconstruction loss against a one-hot target, selected the minimum-loss candidate, and recomputed only that candidate with gradients. The pool used
\[
K\in\{4,8,16,32\},
\]
four paired families, and checkpoints after 1,250, 1,875, and 2,500 updates. This yielded 48 candidates per world. Every continuous generation was decoded to its nearest one-hot cell, so the twelve Voronoi cells formed a total finite output partition.

Both selectors read the same XM-trained candidate pool. Child-validation XM selected the candidate with greatest long-tail child best-of-eight hit. Freedom selection was transductive. It evaluated each candidate on 256 unlabelled parent contexts, retained candidates whose child hit lay within a prospectively selected fraction of the best child hit, and ranked them using a fitted freedom instrument. The primary freedom coefficient was fixed at one. Nonnegative bounded residual terms refined the operational estimate. Two independent calibration streams fitted and tested in opposite directions across 24 calibration worlds. Formula-bank size was selected from 48, 120, and 240. The child-fit gate was selected from $0.97$, $0.98$, and $0.99$. The final instrument was frozen before the thirty final worlds were evaluated.

The frozen instrument used the 240-formula bank and a $0.97$ child-fit gate. Its cross-stream mean calibration gain was $0.068139$. Its structural rate was $0.9167$, exceeding the predeclared $0.80$ requirement that calibration worlds contain at least two eligible candidates separated by one unit of measured log freedom. The balanced-parent hit statistic was excluded from final selection. It was intentionally allowed to read output-profile features computed on unlabelled parent contexts, including the measured freedom score that defined the intervention.

The primary endpoint was best-of-eight hit on the added, balanced parent slice. Child hit was reported separately and was not averaged into the parent endpoint. Across thirty final worlds whose parent-hit statistic was excluded from selection, freedom selection raised mean parent hit from $0.316942$ to $0.389181$. The paired increase was
\[
0.072239
\]
with paired bootstrap 95 percent interval $[0.066296,0.076127]$. Freedom selection improved 29 of 30 worlds and lost one. The exact one-sided paired sign-flip value was $1.86\times10^{-9}$.

\begin{figure}[t]
\centering
\includegraphics[width=.78\linewidth]{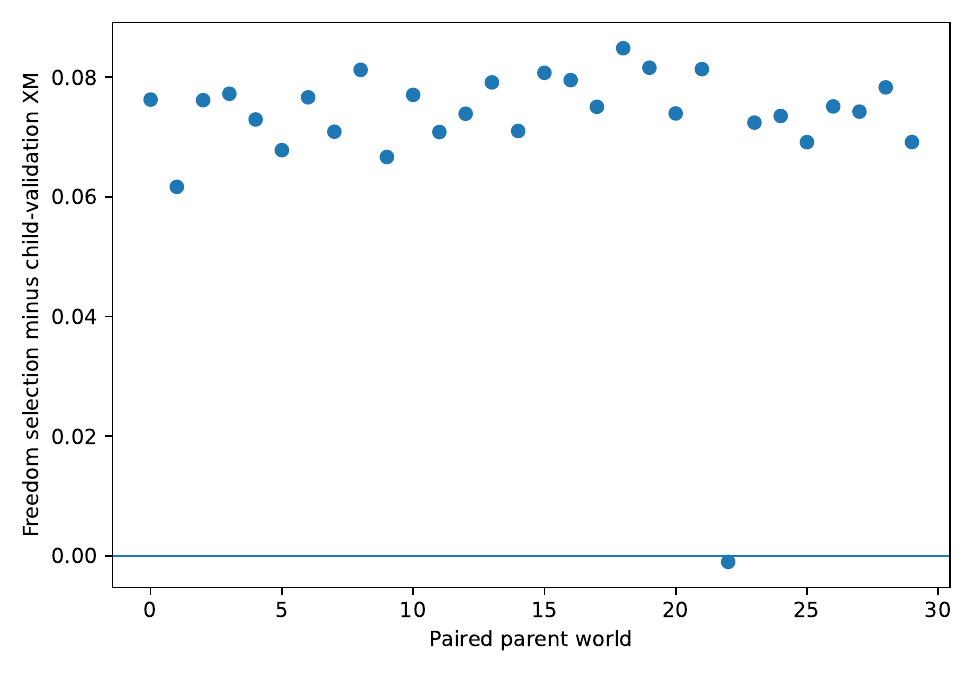}
\caption{Experiment 2. Paired balanced-parent improvement from prospectively calibrated freedom selection over child-validation XM. Freedom improved 29 of 30 final worlds whose parent-hit statistic was excluded from selection.}
\label{fig:eaw2}
\end{figure}

\begin{table}[t]
\centering
\small
\begin{tabular}{lrrr}
\toprule
Quantity & Child-validation XM & Freedom-selected XM & Difference \\
\midrule
Balanced-parent best-of-8 hit & 0.316942 & 0.389181 & $+0.072239$ \\
Child best-of-8 hit & 0.787159 & 0.765462 & $-0.021698$ \\
Active parent outputs & 5.846094 & 8.531641 & $+2.685547$ \\
Mean parent log freedom & 4.033112 & 5.910419 & $+1.877307$ \\
\bottomrule
\end{tabular}
\caption{Experiment 2 final cohort. Freedom selection accepted a small reduction in long-tail child hit and selected XM models with broader valid repertoires and higher balanced-parent hit.}
\label{tab:eaw2}
\end{table}

Freedom selection accepted a $0.021698$ reduction in long-tail child hit while increasing active parent support by $2.685547$ cells and parent log freedom by $1.877307$. Both selectors had access to the same precomputed XM candidate pool, but their selected runs were not individually compute-matched. Child validation selected $K=8$ in all thirty worlds. Freedom selected $K=16$ in 29 worlds and $K=8$ in one. It selected earlier checkpoints on average, $1{,}708$ rather than $2{,}333$ updates, while the mean quantity $K\times\text{updates}$ was $26{,}667$ rather than $18{,}667$. The primary result is therefore a common-pool model-selection comparison, rather than an equal-training-compute comparison between two individual runs.

As a post hoc matched-setting check, I restricted both selectors to the six combinations with $K\in\{16,32\}$ and one fixed checkpoint in $\{1250,1875,2500\}$. Within each combination, both selectors compared the same four family replicates at identical $K$ and update count. Averaging the six paired gains within each world gave $+0.006455$, with paired bootstrap 95 percent interval $[0.005798,0.007079]$. The average was positive in all thirty worlds. This analysis was not predeclared and is reported only as a robustness check.

Child-validation XM preferred candidates specialised to frequent child outputs. Freedom selection chose candidates within three percent of the best child score that retained broader valid repertoires. The balanced parent rewarded those retained possibilities.

\Interp Experiment 2 shows that transductive freedom-enhanced selection can outperform child-validation selection over the same XM candidate pool under a declared frequency shift. Experiment 3 tests freedom-guided selection inside the XM update itself.\footnote{Experiment 3 is in progress using the authors' public XJumpy ImageNet implementation. No result from that experiment is included here.}

\subsection{What the experiments establish}

Experiment 1 identifies a function-level quantity changed by $K$. Experiment 2 uses that quantity to improve selection among XM models. Together they separate mechanism, measurement, and intervention.
\[
K
\longrightarrow
\text{sampled access}
\longrightarrow
\text{deployed permission profile}
\longrightarrow
w(\pi)
\longrightarrow
\text{balanced-parent model selection}.
\]
The first experiment measures the middle relation under sample-based gradient training. The second shows that, among models trained by the same XM procedure, prospectively calibrated freedom selection can outperform child-validation selection when future output frequencies differ. The primary comparison is transductive and common-pool rather than individually compute-matched. Exploration is the mechanism. Freedom is the axis it exposes.

\section{Further predictions}

The theory and experiments suggest several further tests.

\begin{enumerate}[leftmargin=*,label=\arabic*.]
    \item Measure $K$, output support, probability allocation, and freedom separately. A mode count alone cannot test the relation in Example~\ref{ex:skew}.
    \item At fixed correctness, fixed valid-output count, uniform targets, and matched probability balance, models with greater measured local freedom should gain more from increasing $K$. This follows from Theorem~\ref{thm:complement}.
    \item At fixed $K$, mass allocation should approach the $K$-dependent optimum in Theorem~\ref{thm:kkt}. Entropy can balance access across a fixed permitted set, but entropy and freedom agree only under the symmetry conditions of Proposition~\ref{prop:entropy}.
    \item As $K$ grows under nonuniform targets, the nonparametric population optimum should broaden toward full target support. Figure~\ref{fig:nonuniform} plots the exact finite identification prediction. Optimisation error can prevent a trained model from reaching that optimum.
    \item The interaction with model and data scale should depend on whether those axes enlarge the correct permitted region the model can embody. Exploration has little to distinguish when that region is small. It gains value when the model can retain many correct outputs.
    \item A mediation test can distinguish the explanation. If exploration improves generalisation through freedom, the direct association between $K$ and held-out performance should fall after measured freedom is included.
\end{enumerate}

Experiment 1 and Experiment 2 test these relations in a finite output partition where support and completion counts are tractable. Dense softmax, image, video, and language models require a declared operational output vocabulary, a predeclared admissibility rule, or a measured completion proxy. The language and admissibility rule must be fixed before comparisons are drawn.

\section{Related work}

Minimum-over-candidates training has several distinct antecedents. Lloyd quantisation gives the classical centroid geometry behind deterministic squared-loss heads \citep{lloyd1982pcm}. Multiple Choice Learning trained several deterministic structured predictors under an oracle loss \citep{guzmanrivera2012mcl}. Stochastic Multiple Choice Learning adapted this winner-takes-gradient rule to deep ensembles \citep{lee2016smcl}.

Fan, Su, and Guibas introduced Min-of-N for conditional point-cloud generation in a December 2016 preprint, later published at CVPR 2017 \citep{fan2017pointset}. For each input image, they drew $n$ Gaussian latent perturbations, generated $n$ candidate point clouds, and minimised the distance from the target to the closest candidate. This has the same conditional stochastic minimum-over-candidates form as hard Forward XM, though it was developed for one conditional generation problem rather than as a general scaling framework.

IMLE later drew a shared global pool of model samples, matched every datum to its nearest sample, and adjusted the generator through those pairs \citep{li2018imle}. It shares Forward XM's data-to-nearest-generated-sample direction while differing in sample organisation. IMLE also established a maximum-likelihood interpretation under stated conditions. The XM paper retrospectively classifies IMLE as a shared-global-pool instance of end-to-end Forward XM \citep{gladstone2026explorative}.

The XM paper explicitly states that it does not claim the minimum-over-candidates update itself. Its contribution is the broader Explorative Modeling framework. It integrates the update across modern generative architectures, distinguishes Forward and Reverse XM, defines generative expressivity, and studies exploration as a scaling variable across modalities, model sizes, data scales, compute budgets, and generalisation \citep{gladstone2026explorative}. The present paper identifies conditions under which sampled output coverage orders Stack Theory freedom and separates extension size from probability allocation.

Freedom was introduced as a generalisation criterion in 2023 \citep{bennett2023b}. Later work separated it from simplicity \citep{bennett2024b}, repaired and restated the finite optimality result and consolidated the framework in the associated thesis \citep{bennett2025thesis}, and extended freedom to measured continuous languages and neural functions \citep{bennett2026h}. The chronology is relevant. The least-commitment explanation of generalisation came first. Explorative Modeling is a new framework and scaling study built around an older stochastic minimum-over-candidates update. Its deployment-consistent form can reward sampled coverage of policy permissions.

\section{Limitations}

The power-form coverage identity concerns candidates that are conditionally i.i.d. given the context and target. Remark~1 sets out the corresponding joint miss probability for coupled candidates. The identity permits a target-dependent training kernel. The bridge from that kernel to deployed freedom additionally requires deployment consistency. End-to-end XMs meet that condition by construction. The noise-searching continuous hybrids in the XM paper condition each candidate on a different corruption of the datum, so their training kernel need not equal the deployed law. Gladstone, Ji, and Du also leave that coupling-search case for later analysis \citep{gladstone2026explorative}.

The smooth kernel objective studied in the XM appendix has a mixture-likelihood interpretation and retains density information \citep{gladstone2026explorative}. The support limit does not transfer unchanged to that objective.

The closed-form freedom result uses a finite exclusion vocabulary. It treats output cells symmetrically. Richer embodied languages can distinguish two supports of equal cardinality through their refinement structure. In those languages, mode count and freedom separate further.

Raw support can saturate. A finite softmax assigns positive probability to every token. Many continuous generators assign positive density throughout their ambient output space. In either case the raw permission profile is constant and cannot rank models. Support freedom is also discontinuous when an output moves between zero and arbitrarily small positive mass. Empirical work therefore needs a declared finite vocabulary with a predeclared operational admissibility rule, a task-relevant measurable language, or a measured extension proxy. Changing that language or admissibility rule changes the numerical freedom value, so comparison requires one fixed embodied language and rule.

Experiments 1 and 2 measure freedom in fixed finite output partitions. They establish the training and model-selection effects in those settings. Architecture, coupling, optimisation, data, and the chosen embodied language can alter the effect elsewhere. Dense image, video, and language generators still need a fixed operational vocabulary or a measured completion proxy before the same claim can be tested.

Experiment 2 was designed to expose a long-tail-to-uniform frequency shift. Its freedom selector read unlabelled parent contexts, while the child-validation selector did not. The same 240 deployment draws were used to form the selector's parent-context features and the reported parent-hit estimate, so selection and endpoint noise were not independent. The primary selectors could also choose different values of $K$ and different checkpoints. The result therefore establishes a transductive common-pool selection advantage in this synthetic shift. It does not establish an equal-compute improvement to the XM training update. The post hoc matched-setting check reduces the compute concern, but an independent replication should split selector-probe draws from audit draws and predeclare fixed-$K$, fixed-checkpoint comparisons.

Finally, best-of-$K$ can alter optimisation even when support coverage is no longer limiting. Gladstone, Ji, and Du report gains under weak multimodality and hypothesise reduced gradient conflict \citep{gladstone2026explorative}. The present theory does not exclude that route. It isolates the completion-volume route and states when it is exact.

\section{Conclusion}

Exploration and freedom occupy different levels of explanation. Exploration is repeated candidate sampling. Freedom is compatible completion volume. Probability mass connects them only after a deployed policy and embodied language have been fixed.

Best-of-$K$ training transforms the mass of every acceptable region by
\[
q\longmapsto1-(1-q)^K.
\]
Under identification loss and uniform finite targets, this transform equals expected distinct valid-output coverage. Under uniform targets and balanced mass, coverage rises strictly with the local freedom factor. For fixed target-support size, weaker local policies gain more from another sample. Under nonuniform targets, the sequence of finite-$K$ population optima trades frequency against support and approaches uniform mass over full target support. Under independent uniform nonempty demands, future compatibility probability is exactly proportional to unseen freedom.

The experiments complete the argument within their declared finite languages. Across 144 faithful sample-based Forward XM runs, larger $K$ increased or saturated the measured deployed permission profile under uniform targets and increased it under the context-dependent law. In thirty final long-tail-to-balanced worlds, prospectively calibrated freedom selection raised balanced-parent hit from $0.316942$ to $0.389181$ over child-validation XM. It won in 29 worlds, selected 2.686 more active parent outputs, and increased mean parent log freedom by 1.877 while accepting a 0.0217 reduction in long-tail child hit. This was a transductive common-pool selector result. The individual selected runs were not compute-matched.

The third pretraining axis therefore scales sampled access to functional freedom when training and deployment use the same candidate law. Parameters and data may enlarge the permitted region a model can embody. Exploration rewards sampled coverage of that region. The fixed embodied language converts the resulting permission profile into freedom. Parameters decide representation. Data decide learning. Freedom decides future compatibility, and exploration is how training comes to care about it.

Explorative Modeling operationalises an older invariant. The third axis has always been freedom.

\appendix
\section{Verification}

The accompanying script \texttt{fig.py} verifies the finite formulas by enumeration or numerical optimisation and regenerates Figures~\ref{fig:balanced} and \ref{fig:nonuniform}. It checks balanced-support monotonicity, marginal complementarity, entropy balancing, the Karush--Kuhn--Tucker solution against random simplex points, the expected-distinct identity, compatibility proportionality, and Example~\ref{ex:skew}. The script \texttt{check.py} independently recomputes every numerical claim added in the experiment section, including the exact sign-flip value and the matched-setting robustness check. The folders \texttt{eaw1} and \texttt{eaw2}\footnote{Codenames for experiments 1 and 2, now meaningless letter salad.} contain the Colab notebooks, code, and result tables for Experiments 1 and 2.

\section*{Acknowledgement}

A large language model assisted with literature search, simulated reviews, LaTeX prep and project organisation.

\bibliographystyle{plainnat}
\bibliography{master_bibliography}

\end{document}